\documentclass{article}

\usepackage{microtype}
\usepackage{graphicx}
\usepackage{subfigure}
\usepackage{booktabs} 
\usepackage{bm}
\usepackage{hyperref}

\newcommand{\E}{\mathbb{E}}

\newcommand{\fu}{f}

\usepackage[accepted]{icml2025}

\usepackage{amsmath}
\usepackage{amssymb}
\usepackage{mathtools}
\usepackage{amsthm}

\usepackage[capitalize,noabbrev]{cleveref}

\theoremstyle{plain}
\newtheorem{theorem}{Theorem}[section]

\newtheorem{lemma}[theorem]{Lemma}

\theoremstyle{definition}

\newtheorem{assumption}[theorem]{Assumption}
\theoremstyle{remark}
\newtheorem{remark}[theorem]{Remark}

\usepackage[textsize=tiny]{todonotes}

\icmltitlerunning{Submission and Formatting Instructions for ICML 2025}

\begin{document}

\twocolumn[
\icmltitle{Preventing Model Collapse via Contraction-Conditioned Neural Filters}



\icmlsetsymbol{equal}{*}

\begin{icmlauthorlist}
\icmlauthor{Zongjian Han}{equal,hzj}
\icmlauthor{Yiran Liang}{equal,lyr}
\icmlauthor{Ruiwen Wang}{wrw}
\icmlauthor{Yiwei Luo}{lyw}
\icmlauthor{Yilin Huang}{hyl}
\icmlauthor{Xiaotong Song}{sxt}
\icmlauthor{Dongqing Wei}{wdq}


\end{icmlauthorlist}

\icmlaffiliation{hzj}{School of Mathematical Sciences, Tonji University, Shanghai, 200092, P.\,R. China}
\icmlaffiliation{lyr}{School of Mathematical Sciences, Nankai University, Tianjin, China
}
\icmlaffiliation{lyw}{Independent Researcher}
\icmlaffiliation{wrw}{Longhua Clinical College, Shanghai University of Traditional Chinese Medicine, Shanghai, China}
\icmlaffiliation{hyl}{Dundee International Institute, Central South University, Changsha, China}
\icmlaffiliation{sxt}{School of Mathematical Sciences, Shanghai Jiao Tong University, Shanghai, China}
\icmlaffiliation{wdq}{School of Life Sciences and Biotechnology, Shanghai Jiao Tong University, Shanghai, China}

\icmlcorrespondingauthor{Xiaotong Song}{sxtong1997@sjtu.edu.cn}
\icmlcorrespondingauthor{Dongqing Wei}{dqwei@sjtu.edu.cn}

\icmlkeywords{Machine Learning, ICML}

\vskip 0.3in
]



\printAffiliationsAndNotice{\icmlEqualContribution} 

\begin{abstract}
This paper presents a neural network filter method based on contraction operators to address model collapse in recursive training of generative models. Unlike \cite{xu2024probabilistic}, which requires superlinear sample growth ($O(t^{1+s})$), our approach completely eliminates the dependence on increasing sample sizes within an unbiased estimation framework by designing a neural filter that learns to satisfy contraction conditions. We develop specialized neural network architectures and loss functions that enable the filter to actively learn contraction conditions satisfyiung Assumption \ref{as1} in exponential family distributions, thereby ensuring practical application of our theoretical results. Theoretical analysis demonstrates that when the learned contraction conditions are satisfied, estimation errors converge probabilistically even with constant sample sizes, i.e., $\limsup_{t\to\infty}\mathbb{P}(\|\mathbf{e}_t\|>\delta)=0$ for any $\delta>0$. Experimental results show that our neural network filter effectively learns contraction conditions and prevents model collapse under fixed sample size settings, providing an end-to-end solution for practical applications.
\end{abstract}


\section{Introduction}

\subsection{The Challenge of Model Collapse in the Age of Synthetic Data}

The unprecedented scaling of large language models \cite{achiam2023gpt4,touvron2023llama} has created an insatiable demand for training data, far exceeding the available supply of high-quality, human-generated content \cite{villalobos2024data}. To address this gap, researchers and practitioners increasingly turn to \textbf{synthetic data}—data generated by the models themselves. However, this practice poses a critical risk: as synthetic data proliferates online and is inevitably incorporated into future training cycles, it can trigger a degenerative process known as \textbf{model collapse} \cite{shumailov2024collapse}.

Model collapse refers to the progressive degradation in model performance and diversity when generative models are iteratively trained on their own synthetic outputs \cite{shumailov2024collapse,alemberger2024mad}. As illustrated in Figure~\ref{mx2}, this recursive process causes the model to gradually forget the original data distribution, leading to a final model that generates homogeneous and often meaningless outputs.

\subsection{Foundations: Models, Parameters, and Recursive Training}

To establish a precise foundation for our work, we first define the key concepts:

\begin{itemize}
    \item \textbf{Model}: In this paper, we consider a \textbf{parametric generative model} $\mathbb{P}_{\bm{\theta}}$, which is a probability distribution over data points $\bm{x} \in \mathcal{X}$, completely determined by a parameter vector $\bm{\theta} \in \bm{\Theta}$. For example, in a Gaussian model, $\bm{\theta}$ would represent the mean and covariance parameters.

    \item \textbf{Model Parameters}: The vector $\bm{\theta}$ that defines the model's behavior. The goal of training is to estimate these parameters from data. We denote the \textbf{true parameters} that generate real data as $\bm{\theta}^*$, and the estimated parameters at step $t$ as $\widehat{\bm{\theta}}_t$.

    \item \textbf{Recursive Training Workflow}: The standard process that leads to model collapse follows this iterative pattern:
    \begin{enumerate}
        \item Start with real data $\mathcal{D}_0 = \{\bm{x}_{0,i}\}_{i=1}^{n} \sim \mathbb{P}_{\bm{\theta}^*}$
        \item For $t = 1, 2, \ldots, T$:
        \begin{itemize}
            \item Train model $\mathbb{P}_{\widehat{\bm{\theta}}_t}$ on dataset $\mathcal{D}_{t-1}$
            \item Generate new synthetic data $\mathcal{D}_t = \{\bm{x}_{t,i}\}_{i=1}^{n_t} \sim \mathbb{P}_{\widehat{\bm{\theta}}_t}$
            \item Use $\mathcal{D}_t$ to estimate parameters for the next generation: $\widehat{\bm{\theta}}_{t+1} = \mathcal{M}(\mathcal{D}_t)$
        \end{itemize}
    \end{enumerate}
\end{itemize}

This workflow creates a chain: $\mathbb{P}_{\bm{\theta}^*} \rightarrow \mathbb{P}_{\widehat{\bm{\theta}}_1} \rightarrow \mathbb{P}_{\widehat{\bm{\theta}}_2} \rightarrow \cdots \rightarrow \mathbb{P}_{\widehat{\bm{\theta}}_T}$, where each model is trained on data from the previous generation.
\begin{figure}[ht]
\vskip 0.2in
\begin{center}
\centerline{\includegraphics[width=\columnwidth]{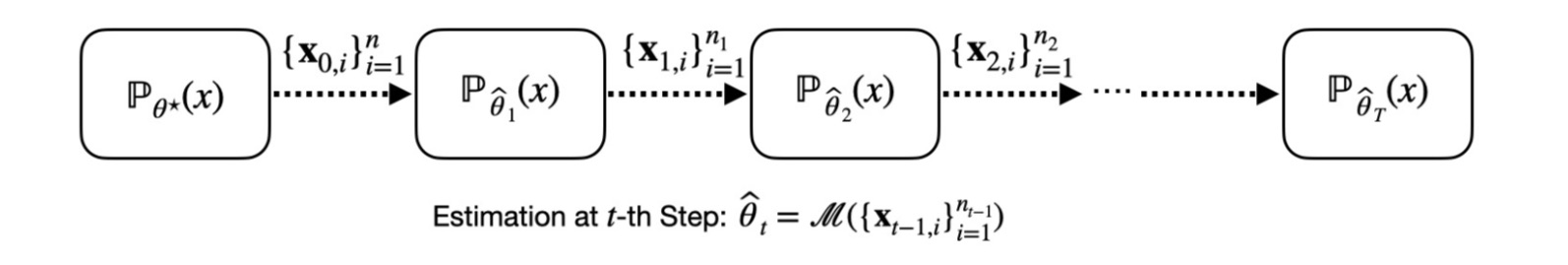}}
\caption{A General Framework for Recursive Training with Fully Synthetic Data.\cite{xu2024probabilistic}}
\end{center}
\vskip -0.2in
\end{figure}
\subsection{A Probabilistic Foundation: Random Walks in Parameter Space}

To demystify why this recursive workflow leads to collapse, \cite{xu2024probabilistic} provide a powerful probabilistic perspective. They conceptualize recursive training as a \textbf{random walk of the model parameters} in the parameter space. This framework elegantly captures the core of the problem:

\begin{itemize}
    \item \textbf{The State}: The estimated parameter $\widehat{\bm{\theta}}_t$ at generation $t$ is the current position of the walker.
    \item \textbf{The Step}: The transition from $\widehat{\bm{\theta}}_{t-1}$ to $\widehat{\bm{\theta}}_t$ is a random step. The direction is determined by the randomness in the finite sample $\mathcal{D}_{t-1}$ drawn from the previous model.
    \item \textbf{The Step Size}: Crucially, the variance of this step—its expected size—is governed by the \textbf{sample size} $n_t$. A larger $n_t$ yields a more precise estimate and a smaller step ($ \|\widehat{\bm{\theta}}_{t} - \widehat{\bm{\theta}}_{t-1}\|_{2} \propto 1/\sqrt{n_t} $).
    \item \textbf{The Drift}: The estimation procedure $\mathcal{M}$ can introduce bias, acting as a consistent force that "drags" the random walk in a particular direction, thereby accelerating its divergence.
\end{itemize}

\begin{figure}[ht]
\vskip 0.2in
\begin{center}
\centerline{\includegraphics[width=\columnwidth]{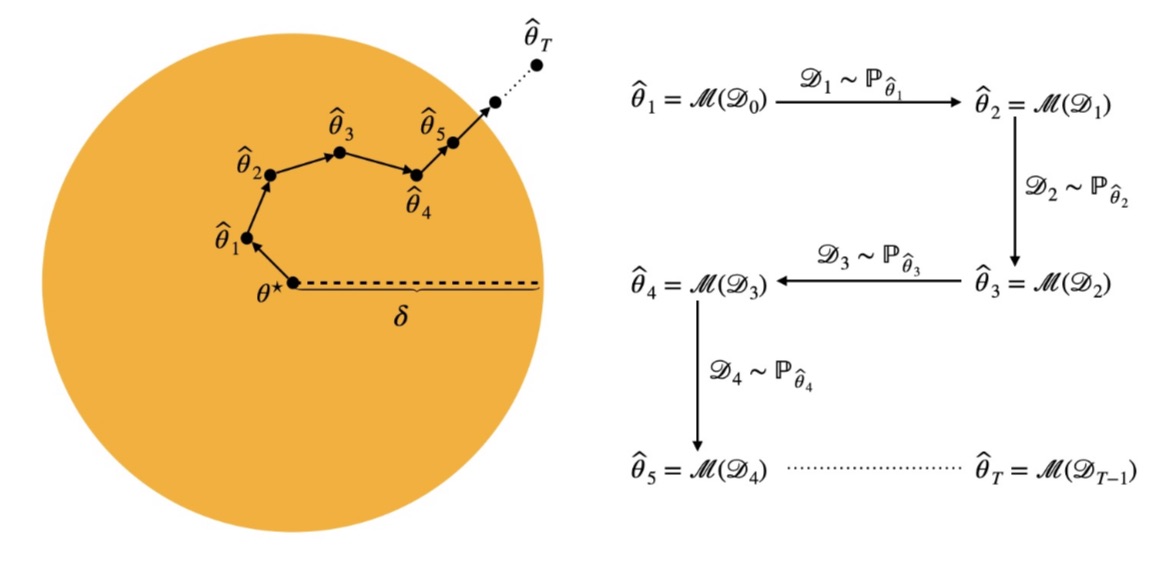}}
\caption{This image explains the principle of model collapse from the perspective of random walks. \cite{xu2024probabilistic}}
\label{mx1}
\end{center}
\vskip -0.2in
\end{figure}

Within this framework, model collapse is the inevitable consequence of this random walk \textbf{drifting away from the true parameter} $\bm{\theta}^*$. \cite{xu2024probabilistic} rigorously show that with a fixed sample size ($n_t = n$), the cumulative error diverges: $\lim_{T\to\infty} \mathbb{E}[(\widehat{\bm{\theta}}_{T} - \bm{\theta}^*)^2] \to \infty$, and the model's diversity vanishes with high probability.

\subsection{The Prevailing Solution: A Theoretically Sound but Practically Limiting Strategy}

The random walk perspective leads to a natural solution: to prevent the walk from straying too far, one must \textbf{progressively reduce the step size}. \cite{xu2024probabilistic} derive the precise conditions for this, proving that to prevent collapse:

\begin{itemize}
    \item For \textbf{unbiased estimators}, the sample size must grow at a \textbf{superlinear rate}, i.e., $n_t = O(t^{1+s})$ for some $s > 0$.
    \item For \textbf{biased estimators}, an even \textbf{faster growth rate} is required, as the bias systematically accelerates the divergence.
\end{itemize}

\begin{figure}[ht]
\vskip 0.2in
\begin{center}
\centerline{\includegraphics[width=\columnwidth]{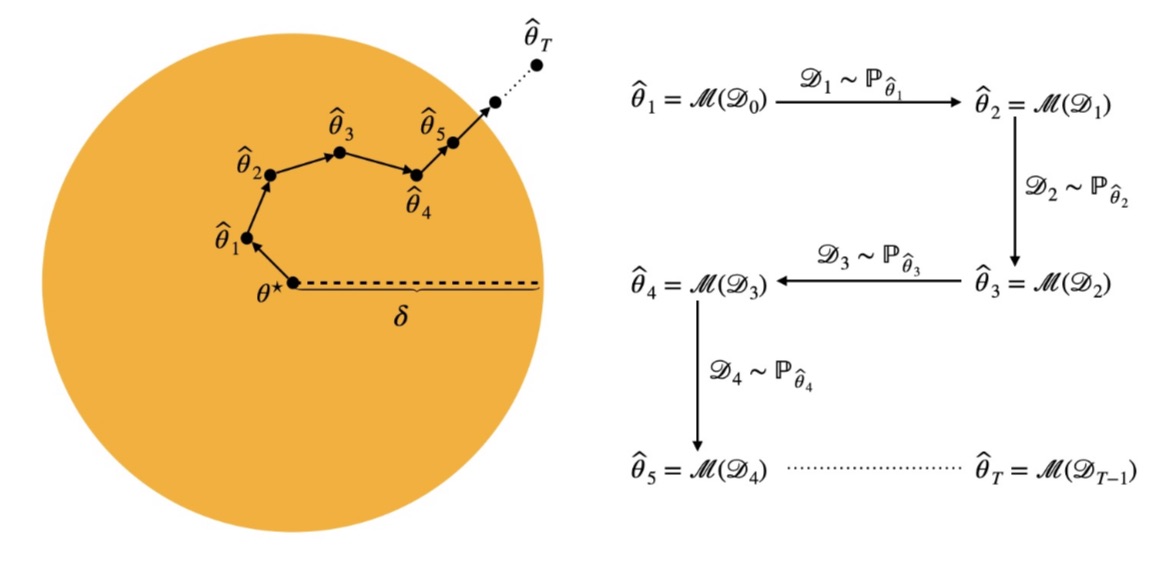}}
\caption{This image explains the principle of model collapse from the perspective of random walks.\cite{xu2024probabilistic}}
\label{mx2}
\end{center}
\vskip -0.2in
\end{figure}
While this constitutes a critical theoretical milestone, the proposed solution presents a \textbf{prohibitive practical barrier}. A superlinear growth schedule implies that computational and storage costs escalate rapidly with each generation. Training over many iterations would become computationally infeasible, rendering this strategy unsuitable for the long-term, sustainable evolution of generative models in real-world scenarios.

\subsection{Our Approach: Steering the Random Walk with Learned Contraction Filters}

We propose a paradigm shift from \emph{containing} the random walk to \emph{actively steering} it. Instead of using exponentially more data to shrink the step size, we introduce an \textbf{neural filter} that learns to correct the walk's direction, ensuring it consistently converges toward the true parameter.

\textbf{Our Enhanced Workflow:} Integrates this filter into the recursive training process:
\begin{enumerate}
    \item Start with real data $\mathcal{D}_0 \sim \mathbb{P}_{\bm{\theta}^*}$
    \item For $t = 1, 2, \ldots, T$:
    \begin{itemize}
        \item Generate candidate synthetic data $\mathcal{D}_t^{\text{candidate}} \sim \mathbb{P}_{\widehat{\bm{\theta}}_t}$
        \item Apply neural filter $g_\phi$ to select a subset $\mathcal{D}_t \subset \mathcal{D}_t^{\text{candidate}}$
        \item Estimate next parameters: $\widehat{\bm{\theta}}_{t+1} = \mathcal{M}(\mathcal{D}_t)$
        \item Update filter parameters $\phi$ to enforce contraction conditions
    \end{itemize}
\end{enumerate}

Our core innovation is to model the error dynamics and enforce a \textbf{contraction condition}. 
When this condition is met, our theoretical analysis (Theorems \ref{thm1} and \ref{thm3}) guarantees that the error process converges in probability. This holds \textbf{even when the sample size $n_t$ is not superlinear}, effectively breaking the superlinear growth requirement established by \cite{xu2024probabilistic}.

\begin{figure}[ht]
\vskip 0.2in
\begin{center}
\centerline{\includegraphics[width=\columnwidth]{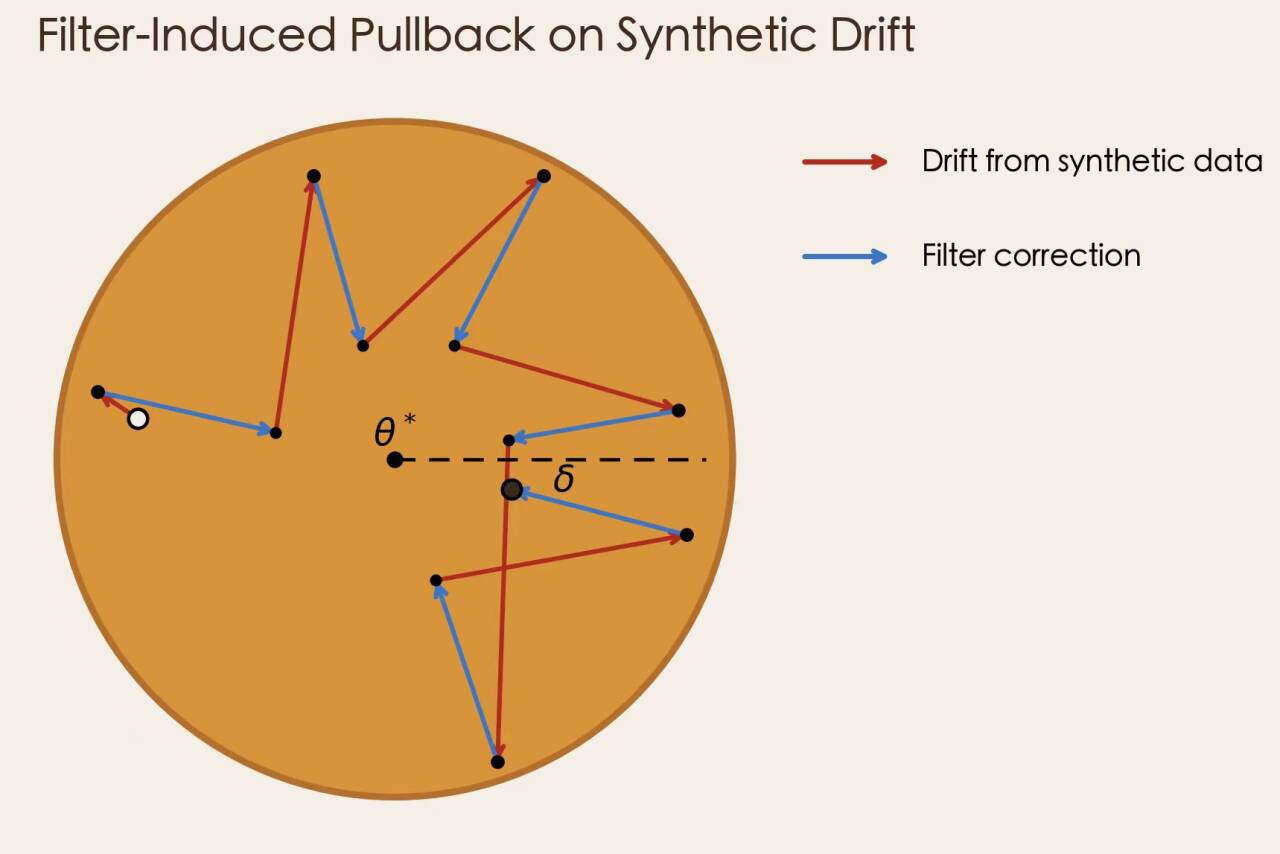}}
\caption{This figure explains the working mechanism of the filter: by continuously pulling the parameters back to the origin, it prevents the model from collapsing.}
\label{pb}
\end{center}
\vskip -0.2in
\end{figure}

As illustrated in Figure~\ref{pb}, our filter acts as a guidance system, continually pulling the wandering parameter estimate back towards the basin of attraction around $\bm{\theta}^*$.

\textbf{Methodology:}
After we establish the mathematical theory under some condition, we give a method to design a neural network as the filter which can force the filter satisfies the condition. We first label the data in the first or the initial few steps by human or by machine, then we use these data to train the filter which make the filter force to the condition under the loss function constraint.

\subsection{Summary of Contributions}

Our work bridges the gap between the theoretical understanding of model collapse and a practical, deployable solution. The main contributions are:

\begin{itemize}
    \item \textbf{A New Theoretical Framework:} We introduce a contraction-operator-based theory that provides a sufficient condition for preventing model collapse with \textbf{constant sample sizes}, moving beyond the superlinear growth requirement of \cite{xu2024probabilistic}.
   
    \item \textbf{An End-to-End Algorithmic Solution:} We design a specialized neural network architecture that acts as a data filter. This network is trained with a novel combined loss function to actively \textbf{learn the contraction conditions} required by our theory, ensuring theoretical guarantees are met in practice.
   
    \item \textbf{Integration of Theory and Practice:} We provide a complete framework that translates theoretical stability conditions into a learnable objective for a neural network, ensuring our method is both principled and practical.
   
    \item \textbf{Empirical Validation:} We demonstrate through experiments that our neural filter successfully prevents model collapse in exponential family distributions under fixed sample size settings, validating its effectiveness and superiority.
\end{itemize}

\section{Mathematical Modeling and Deduction}
\label{sec:methodology}

\subsection{Mathematical Modeling}
\label{subsec:modeling}
We aim to train a filter to collect data to avoid model collapse of unbiased estimation, which has a function to pull the point back to the origin, which is abstracted as the map $B$.

Consider the nonlinear stochastic difference system:
\begin{equation}
    \mathbf{e}_{t+1} = A(\mathbf{e}_t)\mathbf{e}_t + \boldsymbol{\xi}_t', \quad \text{where } A(\mathbf{e}_t) = I - B(\mathbf{e}_t)\label{eq1}
\end{equation}

and \(B(\mathbf{e}_t)\) is a state-dependent contraction operator.

We make the following assumptions to characterize the role of filters in stochastic systems with model collapse.
{\bf We use the error generated by the estimator to produce the random walk:}
  \begin{assumption}\label{as2}
        { (Noise Properties)}: \\
    The noise process \(\{\boldsymbol{\xi}_t'\}\) satisfies
    \begin{enumerate}
        \item  \(\mathbb{E}[\boldsymbol{\xi}_t' | \mathcal{F}_t] = 0\), where \(\mathcal{F}_t = \sigma(\mathbf{e}_0, \boldsymbol{\xi}_0', \ldots, \boldsymbol{\xi}_{t-1}')\) is the filtration.
        \item  there exists a sequence \(\sigma^2_t > 0\) and $\lim\limits_{t\to\infty}\sigma^2_{t}=0$ such that \(\mathbb{E}[\boldsymbol{\xi}_t'^T P \boldsymbol{\xi}_t' | \mathcal{F}_t] \leq \sigma_t^2\) almost surely.
    \end{enumerate}
    \end{assumption}
    \begin{remark}
        Actually this is reasonable, by \cite{xu2024probabilistic}, when the sample growth (here the growth can be in any speed), the upper bound of the step size in the random walk of parameter will turn to zero.
    \end{remark}
    {\bf We use a functional $c$ to control the map $A$:}
    \begin{assumption}{ (contraction condition)} \label{as1}

        There exists a symmetric positive definite matrix \(P \succ 0\) and a continuous function \(c: \mathbb{R}^p \to [0,1)\) such that:
    \[
    A(\mathbf{e})^T P A(\mathbf{e}) \preceq (1 - c(\mathbf{e})) P, \quad \forall \mathbf{e} \in \mathbb{R}^p
    \]
    \end{assumption} 
    \begin{remark}
        The function $C$ can be seen as the shrinking strength, where the strength increases with the increase of error seen in figure~\ref{sf}. 
    \end{remark}    
    \begin{figure}[ht]
\vskip 0.2in
\begin{center}
\centerline{\includegraphics[width=\columnwidth]{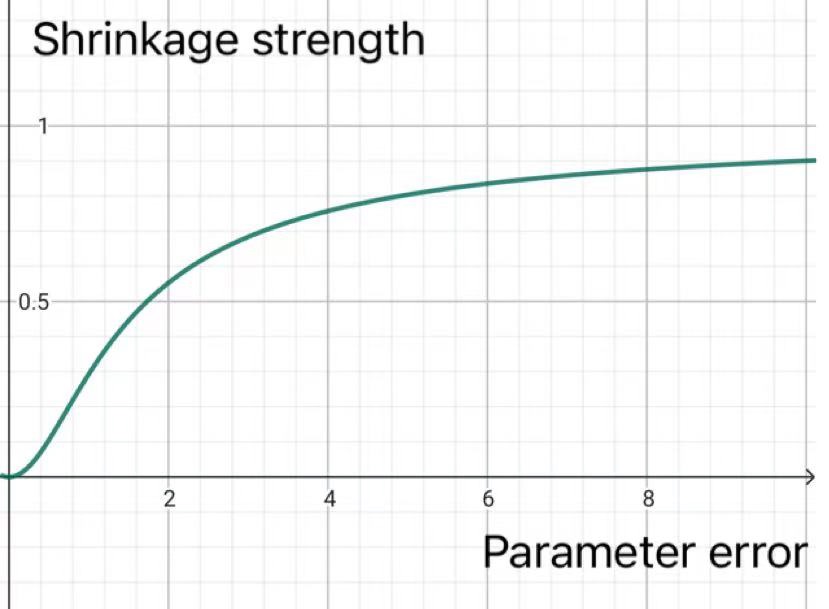}}
\caption{This figure show the relationship between the contraction strength and the vector distance from the origin which is the parameter error.}
\label{sf}
\end{center}
\vskip -0.2in
\end{figure}
      {\bf We introduce a convex function $\fu$ to regulate the contraction function $c$ which can ensure that the farther away from the origin, the stronger the pulling back effect:}
\begin{assumption}\label{as3}
    {(Properties of the Contraction Function)}: 
    
        There exists a convex function \(\fu: \mathbb{R}_{\geq 0} \to \mathbb{R}_{\geq 0}\) with \(\fu(0) = 0\), \(\fu(r) > 0\) for \(r > 0\), and
        $$c(\mathbf{e}) V(\mathbf{e}) \geq \fu(V(\mathbf{e}))
        ,$$
        here $V(\mathbf{e}):=\mathbf{e}^TP\mathbf{e}$ is the Lyapunov function.
\end{assumption}
The function $c$ and $f$ are exist,

\textbf{Example:}
     We can just take $c:\mathbf{e}\mapsto 1-(V(\mathbf{e})+1)^{-\frac{1}{2}}$ and $\fu:r\mapsto r(1-(r+1)^{-\frac{1}{2}})$ can satisfies the condition. 
\begin{remark}
    Here the function $f$ is used to regular the function $C$, and we need not to worry about the wether the condition can be satisfied, we will construct a loss function to make the filter converge the condition in section \ref{ls}.
\end{remark}

\subsection{Main Theoretical Results}
Here we prove that this random system with a pullback effect, when the square of the norm of the step length of the random walk is controlled by a sequence that converges to 0, will eventually enter any neighborhood of the origin.
\begin{theorem}\label{thm1}
If a stochastic difference system~\ref{eq1} satisfies Assumption~\ref{as1}, ~\ref{as2} and ~\ref{as3}, then for each $\delta>0$    $\limsup\limits_{t \to \infty} \mathbb{P}(\|\mathbf{e}_t\| > \delta)=0$ \label{hl}, {\bf which means the model will not collapse}.
\end{theorem}

\subsection{The Estimation of Convergence Rate}
We make the following assumptions on the growth of $\fu$ and the convergence rate of noise, in order to provide the convergence rate of this stochastic system with a pullback effect.

\begin{assumption}[Function Properties]
\label{assump:function}
The function $f: \mathbb{R}_{\geq 0} \to \mathbb{R}_{\geq 0}$ satisfies:
 There exist constants $c_1, c_2 > 0$ and $x_0 > 0$ such that for all $0 < x \leq x_0$: $$c_1 x^p \leq f(x) \leq c_2 x^p .$$
\end{assumption}

\begin{assumption}[Noise Properties]
\label{assump:noise}
The noise sequence $\{\sigma_t^2\}$ satisfies:

 $$\sigma_t^2 = O(t^{-\beta})$$ for some $\beta > 0$.
\end{assumption}

\begin{theorem}[Convergence Rates for Recurrence Inequality]
\label{thm2}
Under Assumptions \ref{assump:function} and \ref{assump:noise}, the sequence $\{x_t\}$ satisfies the following convergence rates:

1. If $p = 1$, then $x_t = O\left(\max\left(e^{-ct}, t^{-\beta}\right)\right)$ for some $c > 0$, where $c = -\log(1-c_1) > 0$.

2. If $p > 1$, then $x_t = O\left(\max\left(t^{-\frac{1}{p-1}}, t^{-\frac{\beta}{p}}\right)\right)$
\end{theorem}
\subsection{The Combination with the Work of Predecessors}
Here we use the assumption and consequence from \cite{xu2024probabilistic} to prove that when the filter exists, sample size $n_t$ need not grow superlinearly. 
\begin{assumption}\cite{xu2024probabilistic}\label{initass}
For a class of parametric generative models $\mathcal{P}=\{\mathbb{P}_{\theta}:\theta\in\Theta\}$,
let $\mathcal{D}=\{\boldsymbol{x}_{i}\}_{i=1}^{t}$ be a dataset generated from $\mathbb{P}_{\theta}$
for some $\theta\in\Theta$. Suppose that $\widehat{\boldsymbol{\theta}}=\mathcal{M}(\mathcal{D})$
is an estimate of $\boldsymbol{\theta}$ obtained under the estimation scheme $\mathcal{M}$.
Assume there exist constants $C_{1},C_{2},\gamma>0$ and a positive diverging sequence $r(t)$
such that for all $t\geq 1$ and any $\delta>0$, we have
\begin{align*}
\sup_{\theta\in\Theta}\mathbb{P}\left(\|\widehat{\boldsymbol{\theta}}-\boldsymbol{\theta}\|_{2}\geq\delta\right)
\leq C_{1}\exp(-C_{2}r(t)\delta^{\gamma}),
\end{align*}
where the probability is taken over the randomness of $\mathcal{D}$ generated i.i.d. from $\mathbb{P}_{\theta}(\boldsymbol{x})$.
\end{assumption}

\begin{theorem}\label{thm3}
   Suppose that $\widehat{\boldsymbol{\theta}}=\mathcal{M}(\mathcal{D})$ is an estimate of
$\boldsymbol{\theta}$ with $\mathcal{D}=\{\boldsymbol{x}_{i}\}_{i=1}^{t}\sim\mathbb{P}_{\boldsymbol{\theta}}$
satisfying Assumption \ref{initass} with $r(t)=t^{\kappa}$, where $\kappa>0$. Then for each $\delta>0$, $$\limsup\limits_{t \to \infty} \mathbb{P}(\|\mathbf{e}_t\| > \delta)=0,$$
{\bf which means the model will not collapse}.
\end{theorem}
\begin{remark}
    Theorem~\ref{thm3} tell us that when we have a filter which satisfies Assumption~\ref{as1} and~\ref{as3}, then we only need a lower growth speed thar can make sure the model would not  collapse.
\end{remark}

\section{Neural Architecture Design}
\label{subsec:architecture}
We propose a neural network-based filter designed to select data points for parameter estimation in exponential family distributions. When we constructed $c,\fu,P$ in Assumption~\ref{as1} and \ref{as3}, the filter is trained via gradient descent to satisfy  contraction condition (Assumption~\ref{as1}) on the estimation error while maintaining classification accuracy against human labels. Our solution includes the network architecture, loss functions, and training procedure.. 
\newcommand{\bt}{\{\bm{x}_{\text{est},i}\}_{i=1}^{n}}
\newcommand{\tte}{\hat{\theta}_\text{est}}
 \newcommand{\bg}{\{\bm{x}_{\text{good},i}\}_{i=1}^{m}}

Let $\mathbb{P}_{\theta}$ be a distribution and we have samples $\bt$ and a unbiased estimation $\mathcal{M}$.  Define $\tte:=\mathcal{M}(\bt)$.

For example for exponential family distributions
\begin{equation}
    f(x|\theta) = h(x) \exp\left(\theta^T T(x) - \Phi(\theta)\right) \label{exp}
\end{equation}

where $h(x)$ is the base measure, $T(x)$ is the sufficient statistic, and $\Phi(\theta)$ the log-partition function.
we can use 
 \[
        \hat{\theta}_\text{est} = (\nabla \Phi)^{-1}(\bar{T}_\text{est})
        \]
to approximate its parameters, where 
        \[
        \bar{T}_\text{est} = \frac{1}{n} \sum_{i=1}^n T(x_i)
        \]
is the sufficient statistic for the sample.

\subsection{Problem Statement}

Given a dataset \(D = \{x_i\}_{i=1}^N\) from an distribution with unknown true parameter \(\theta_{\text{true}}\), we design a filter \(g_\phi(\cdot)\) :
\begin{enumerate}
    \item Outputs selection probabilities \(w_i \in [0,1]\) for each data point based on (PCA-reduced features \(z_i\), if we use PCA)  data component.
    \item Satisfies the contraction condition:
    \begin{equation}\label{eq:contract}
        \mathbf{e}_{\text{new}}^\top P \mathbf{e}_{\text{new}} \leq (1 - c(\mathbf{e}_{\text{est}})) \mathbf{e}_{\text{est}}^\top P \mathbf{e}_{\text{est}}
    \end{equation}
    where \(\mathbf{e}_{\text{est}} = \theta_{\text{est}} - \theta_{\text{good}}\), \(\mathbf{e}_{\text{new}} = \theta_{\text{new}} - \theta_{\text{good}}\), \(P \succ 0\) is a symmetric positive definite matrix, and \(c(\mathbf{e})\) is a continuous function.
    \item Minimizes classification error against human-provided labels \(y_i \in \{0,1\}\) (0: bad sample, 1: good sample).
\end{enumerate}

\textbf{Key Components and Assumptions}
\begin{enumerate}
\item {Data Preparation}
\begin{itemize}
    \item Dataset: \(D = \{x_i\}_{i=1}^N\) with labels \(y_i\) which are labeled as good sample. (The label can be formed from human or the machine. In our experiment we use machine to label sample.)
    \item We use PCA to raw features to obtain reduced-dimensional features \(z_i \in \mathbb{R}^d\).
\end{itemize}

\item {Parameter Estimation}
\begin{itemize}
    \item \textbf{Initial estimate} \(\theta_{\text{est}}\):
    \[
    \theta_{\text{est}} = \mathcal{M}(\bt)
    \]
    \item \textbf{Reference estimate} \(\theta_{\text{good}}\) (using good samples only):
    \begin{align*}
 \theta_{\text{good}} = \mathcal{M}(\bg)
    \end{align*}
    \item \textbf{Error vectors}: \(\mathbf{e}_{\text{est}} = \theta_{\text{est}} - \theta_{\text{good}}\), \(\mathbf{e}_{\text{new}} = \theta_{\text{new}} - \theta_{\text{good}}\)
\end{itemize}

\item {Filter Operation}
The filter defines a \textit{pullback operator} \(B\) implicitly through data selection:
\[
A = I - B, \quad \mathbf{e}_{\text{new}} = A \mathbf{e}_{\text{est}}
\]
where \(A\) is the matrix satisfying the contraction condition \eqref{eq:contract}.
\end{enumerate}

\subsection{ Neural Architecture Design}\label{nnd}
\begin{itemize}
    \item \textbf{Architecture}: Multi-layer Perceptron (MLP)
    \item \textbf{Input}: Data or PCA-reduced features \(z_i \in \mathbb{R}^d\)
    \item \textbf{Hidden layers}: \(\geq 1\) layer with ReLU activation (e.g., dimension 128)
    \item \textbf{Output}: Single node with sigmoid activation \(\rightarrow w_i = g_\phi(z_i)\)
    \item \textbf{Parameters}: Weights and biases denoted as \(\phi\)
\end{itemize}

\textbf{Loss Function}
\label{ls}
Total loss combines classification and contraction objectives:
\begin{equation}\label{eq:total_loss}
L_{\text{total}}(\phi) = L_{\text{class}}(\phi) + \lambda L_{\text{contract}}(\phi)
\end{equation}
where \(\lambda > 0\) is a hyperparameter.
\begin{enumerate}
    \item Classification Loss
Binary cross-entropy with human labels:
\small{
\begin{equation}
L_{\text{class}}(\phi) = -\frac{1}{N} \sum_{i=1}^N \Big[ y_i \log(g_\phi(z_i)) + (1 - y_i) \log(1 - g_\phi(z_i)) \Big]
\end{equation}}

\item {Contraction Loss}
Hinge loss enforcing \eqref{eq:contract}:
\begin{equation}
L_{\text{contract}}(\phi) = \max\left(0, \, \mathbf{e}_{\text{new}}^\top P \mathbf{e}_{\text{new}} - (1 - c(\mathbf{e}_{\text{est}})) \mathbf{e}_{\text{est}}^\top P \mathbf{e}_{\text{est}} \right)
\end{equation}
where:
\begin{itemize}
    \item \(\mathbf{e}_{\text{est}}\) is fixed during training
    \item \(\mathbf{e}_{\text{new}}\) depends on \(\phi\) via \(\theta_{\text{new}}\)
    \item \(c(\mathbf{e})\) is predefined (e.g., \(c(\mathbf{e}) = \alpha \|e\|^2_2\))
\end{itemize}
\end{enumerate}

\subsection{Training Procedure}
We train the network as follows, which is flexible and depend on the design of researcher. Here we use the exponential family distributions for example.

\begin{algorithm}[H]
\caption{Filter Training}
\begin{algorithmic}[1]
\STATE \textbf{Initialize:}
    \begin{itemize}
        \item Compute \(\theta_{\text{est}}\), \(\theta_{\text{good}}\) from \(D\)
        \item Extract PCA features \(z_i\) for all \(x_i\)
        \item Initialize filter parameters \(\phi\)
    \end{itemize}
\FOR{epoch = 1 to MaxEpochs}
    \STATE \textbf{Forward Pass:}
        \begin{itemize}
            \item Compute weights: \(w_i = g_\phi(z_i)\)
            \item Compute weighted sufficient statistic:
            \[
            \bar{T}_{\text{new}} = \frac{\sum_{i=1}^N w_i T(x_i)}{\sum_{i=1}^N w_i}
            \]
            \item Solve for \(\theta_{\text{new}}\):
            \[
            \theta_{\text{new}} = (\nabla \Phi)^{-1}(\bar{T}_{\text{new}}) \quad \text{(analytic/numerical)}
            \]
            \item Compute error: \(\mathbf{e}_{\text{new}} = \theta_{\text{new}} - \theta_{\text{good}}\)
        \end{itemize}
    \STATE \textbf{Loss Calculation:}
        \[
        L_{\text{total}} = L_{\text{class}} + \lambda L_{\text{contract}}
        \]
    \STATE \textbf{Backward Pass:}
        \[
        \phi \leftarrow \phi - \eta \nabla_\phi L_{\text{total}} \quad \text{(via Adam optimizer)}
        \]
\ENDFOR
\end{algorithmic}
\end{algorithm}

\subsection{Convergence Analysis}
\textbf{Existence of Solution}
Assume:
\begin{itemize}
    \item \(\exists \phi^*\) such that \(L_{\text{class}}(\phi^*) = 0\) and \(L_{\text{contract}}(\phi^*) = 0\)
    \item Neural network has sufficient capacity (Universal Approximation Theorem)
\end{itemize}
Then gradient descent may converge to \(\phi^*\) satisfying the contraction condition.

\textbf{Optimization Convergence}
\begin{itemize}
    \item \(L_{\text{total}} \geq 0\) but non-convex \(\rightarrow\) converges to local minimum
    \item \(L_{\text{contract}}\) provides gradient when contraction is violated:
    \[
    \nabla_\phi L_{\text{contract}} \propto \nabla_\phi \left( \mathbf{e}_{\text{new}}^\top P \mathbf{e}_{\text{new}} \right) \quad \text{if } L_{\text{contract}} > 0
    \]
    \item Monitor \(L_{\text{contract}} \to 0\) during training
\end{itemize}

\section{Experiments}
\label{sec:experiments}
We design experiments to show the effect of filter, we use a neural network to collect data so that we can have enough sample to train the filter. we will compare the offset of parameters between whit and with out filter in three kind of sample growth rate.

\subsection{Experimental Setup}
\label{subsec:setup}
Now we have a set of initial data generated by a normal distribution.

 We first simulate the workflow for several step and we labeled the \(70\%\) of points closest to the true parameters as ‘good’ and the remainder are marked  ‘bad’ . Then we take these data to training the filter. finally we add the filter into the workflow and recode the {\bf expected distance $ \mathbb{E}[(\widehat{\bm{\theta}}_{T} - \bm{\theta}^*)^2]$, between the parameters and the true parameters} when the sample grows. 

\subsection{Training data}
The training data used in this study was entirely generated by simulation, aiming to reproduce the sample quality degradation caused by distribution drift during the iterative parameter estimation process. We construct datasets respectively with four dimension configurations (2 and 3 dimensions), and the data of each dimension is based on the same set of default hyperparameters: The true mean vector is initialized to a constant of 1 ($\mu_0=1.0$), the covariance is determined by the identity matrix (standard deviation  $\sigma_0=1.0$), the number of iteration rounds $n_{\text{rounds}}$, the base sample size of each round $n_{\text{samples}}$ , and the contamination rate shows that the filter effectively guides the generated data toward the true distribution.
In the specific process, in the t-round, 1000 candidate points will be sampled around the current estimate, sorted in ascending order by the Euclidean distance from the true mean, and the closest 1- contamination rate will be marked as the "good sample" (label 1). The rest contamination rate is marked as "bad sample" (label 0). Subsequently, only by making good use of the mean update of the samples $\theta_{t+1}$, systematic offsets are injected round by round and cumulative drifts in real scenarios are created.

\subsection{Experimental Process}

After obtaining the data, the merged dataset is standardized and PCA dimensionality reduction is performed. Then, a two-layer MLP (FilterNet) is used to jointly optimize with binary cross-entropy, contraction constraints, and ESS regularization to learn to distinguish between samples that are "helpful for estimation" and those that "cause offset".
After the training was completed, we divided the iterative estimation process under the same initial conditions into two groups of control experiments: "no filter" and "loaded filter". The former directly updated the parameters by averaging all samples, while the latter first weighted the samples through the trained filter and then updated them.

%

\subsection{Results and Analysis}
Figure~\ref{sy2} shows that the filter will pull the generated data to the real case effectively. And we can see that the parameters tend to converge to the true parameters.  

\begin{figure}[ht]
\vskip 0.2in
\begin{center}
\centerline{\includegraphics[width=\columnwidth]{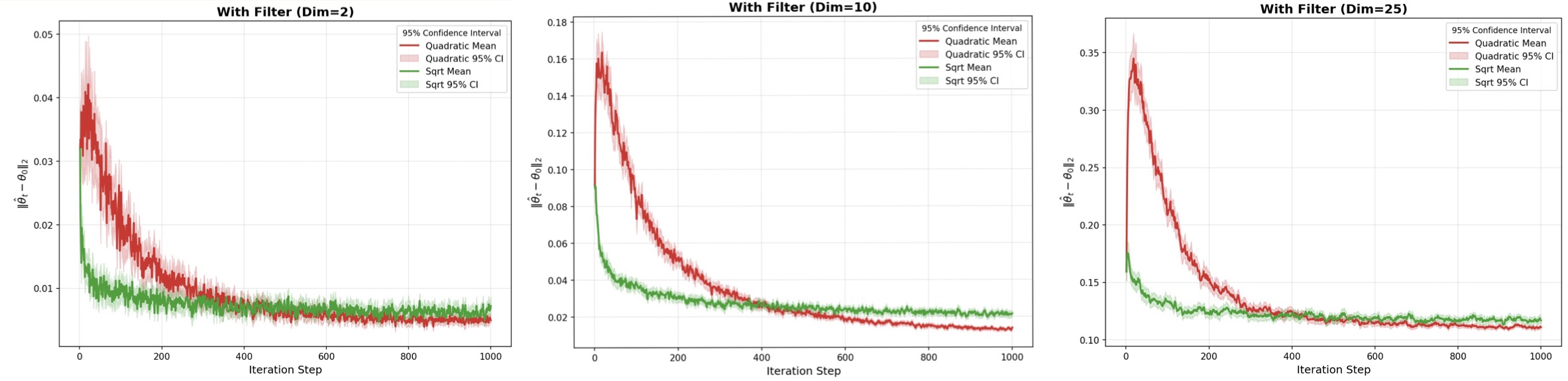}}
\caption{Expectation distance between
the parameters and the true parameters when the sample
grows in $2,10$ and $25$ dimension normal distribution. Here the red line is  the $t^2$ growth and the green line is $t^{\frac{1}{2}}$ growth.}
\label{sy2}
\end{center}
\vskip -0.2in
\end{figure}

\subsection{Ablation Studies}
\label{subsec:ablation}
We compare the {\bf eliminated the dependence on superlinear sample growth and initial parameters} when the sample grows at two different rates, {\bf with and without filters}.

The first row of Figure~\ref{dbsy1-2} shows the behavior of parameter estimates without the use of a filter. Specifically, as the sample size grows at a square root rate, the parameter estimates diverge infinitely away from the true parameter, which is consistent with the theoretical results in \cite{xu2024probabilistic}. 

The second row of Figure~\ref{dbsy1-2} illustrates the effectiveness of the filter in stabilizing the parameter estimates. It reveals that the filter's "pulling effect," which counteracts the divergence of the parameter estimates, becomes more pronounced as the sample size increases. This demonstrates the filter's critical role in ensuring stable and accurate parameter estimation as the data grows.

\begin{figure}[ht]
\vskip 0.2in
\begin{center}
\centerline{\includegraphics[width=\columnwidth]{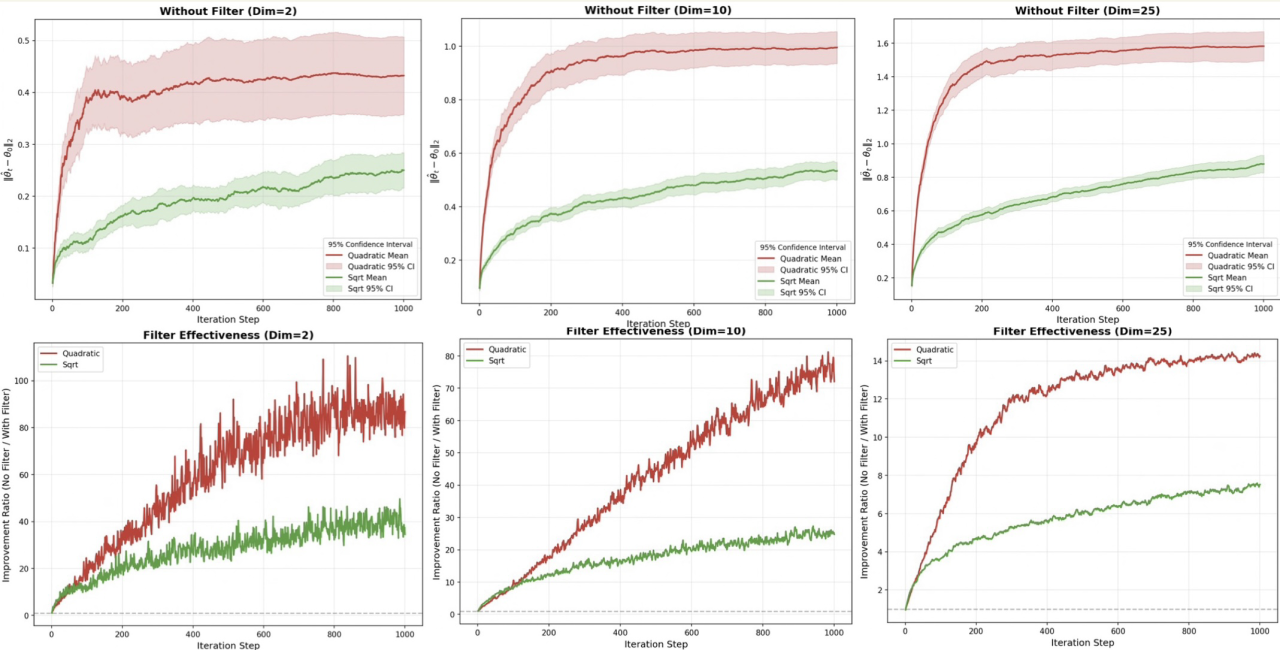}}
\caption{Expectation distance and Improvement ratio(No filter/ with filter) between
the parameters and the true parameters when the sample
grows in  $t^{\frac{1}{2}}$ and $t^{\frac{1}{2}}$ speed without filters in $2$, $10$ and $25$ dimension normal distribution.  Here the red line is  the $t^2$ growth and the green line is $t^{\frac{1}{2}}$ growth. }
\label{dbsy1-2}
\end{center}
\vskip -0.2in
\end{figure}

\section{Discussion}
\label{sec:discussion}

\subsection{ Theoretical Implications}
Our work establishes a novel theoretical framework for preventing model collapse in recursive training of generative models, grounded in contraction operators and Lyapunov stability theory. By formulating the error dynamics as a nonlinear stochastic system and enforcing a contraction condition through a learned neural filter, we demonstrate that model collapse can be avoided even with arbitrary sample growth and arbitrary initial sample. This represents a significant departure from the superlinear sample growth requirement established by \cite{xu2024probabilistic}, offering a more practical and scalable solution for long-term model training.

Our theoretical results (Theorems~\ref{thm1} and~\ref{thm3}) show that the convergence in probability of the estimation error can be achieved under mild assumptions on the noise (which is depend on the work of \cite{xu2024probabilistic}) and contraction function. This implies that the filter can actively steer the parameter estimate back toward the true parameter, effectively acting as a stabilizing controller in the recursive training loop.

To ensure that the mathematical theory works, we have designed a loss function that enables the filter to converge to the hypothesis. After the data labeled under the simulated working flow was used as the filter for training, an excellent screening effect was achieved, effectively preventing the model from collapsing. When the sample is growing superlinearly, although  \cite{xu2024probabilistic} can ensure that the sample parameters do not deviate too far with a sufficient sample size, the requirement for the initial sample size is directly abandoned after adding the filter, and there is an extremely significant improvement compared to not adding the filter.

\subsection{Practical Applications}

The proposed neural filter is particularly relevant in scenarios where high-quality human-generated data is scarce or expensive to obtain. For example:
\begin{enumerate}
    \item {\bf Large Language Models (LLMs):}
    
    As models like GPT-4 and beyond continue to scale, synthetic data will play an increasingly critical role. Our filter can help maintain model quality and diversity over multiple generations of self-training.
    \item {\bf Data Augmentation and Privacy-Preserving Synthesis:}
    
    In domains such as healthcare or finance, where real data is sensitive, synthetic data generated under our framework can be used safely without risking model collapse.
    \item {\bf Continual Learning:}
    
    Our approach can be integrated into continual learning pipelines where models are periodically updated with new synthetic data, ensuring long-term stability.
\end{enumerate}

\subsection{Limitations}
While our method offers a promising solution, several limitations remain:
\begin{enumerate}
    \item We need to know the specific distribution type and the non-specific estimation method of its parameters. We have not yet established a theory for partial estimation.
    \item The filter’s performance depends on the availability of a high-quality reference parameter \(\theta_{\text{good}}\), which may not always be accessible in fully unsupervised settings.
\end{enumerate}

\section{Conclusion}
\subsection{summary}
We conducted mathematical modeling on the filter and identified the sufficient conditions for the convergence of its control system. We also designed a neural network for the filter and incorporated the filter into the neural network. This modification performed very well in the experiment in reducing model collapse and broke through the sample's dependence on superlinear growth and initial parameters.

\subsection{Future work}
In the future, we will consider the theoretical framework when there is partial estimation, and also design neural networks to converge to the theoretical assumptions.

\nocite{xu2024probabilistic}
\nocite{achiam2023gpt4}
\nocite{touvron2023llama}
\nocite{villalobos2024data}
\nocite{shumailov2024collapse}
\nocite{alemberger2024mad}
\bibliography{example_paper}
\bibliographystyle{icml2025}

\newpage
\appendix
\onecolumn
\section{Appendix}

\subsection{proof of Theorem \ref{thm1}}
Notice that $\lim\limits_{t\to\infty}\sigma^2_{t}=0$ which means for each $\varepsilon>0$, there exist $N$ for each $t>N$, we have
\begin{align*}
    \sigma^2_{t}< \varepsilon.
\end{align*}
Now we fix $\varepsilon$ and begin from $t=N+1$.

Taking unconditional expectation in Assumption \ref{as1}:
\[
\E[V(\mathbf{e}_{t+1})] \leq \E[V(\mathbf{e}_t)] - \E[c(\mathbf{e}_t) V(\mathbf{e}_t)] + \varepsilon.
\]
By Assumption \ref{as3}, $c(\mathbf{e}_t) V(\mathbf{e}_t) \geq \fu(V(\mathbf{e}_t))$, so:
\[
\E[V(\mathbf{e}_{t+1})] \leq \E[V(\mathbf{e}_t)] - \E[\fu(V(\mathbf{e}_t))] + \varepsilon. 
\]
Notice that $\fu$ is a convex function, then we have
\begin{align*}
    \E[V(\mathbf{e}_{t+1})] \leq \E[V(\mathbf{e}_t)] - \fu(\E[V(\mathbf{e}_t)]) + \varepsilon.
\end{align*}
Let $x_t:=\E[V(\mathbf{e}_{t+1})]$ and $b:=\varepsilon$ then we get
\begin{align*}
    x_{t+1} \leq  x_{t} - \fu(x_{t+1}) + b.
\end{align*}
then we have
\begin{lemma}
Consider a nonnegative sequence \(\{x_t\}\) satisfying the recurrence inequality:
\[
x_{t+1} \le x_t - \fu(x_t) + b,
\]
where \(b > 0\), and \(\fu\) is a positive convex function (i.e., convex and \(\fu(x) > 0\) for all \(x > 0\)). Assume that \(\fu\) is continuous, and the equation \(\fu(x) = b\) has a largest solution \(L\) (i.e., \(L = \max \{ x \ge 0 : \fu(x) = b \}\)). Then the limit superior of the sequence satisfies:
\[
\limsup\limits_{t \to \infty} x_t \le L.
\]
\end{lemma}

\begin{proof}
Assume, for contradiction, that \(\limsup\limits_{t \to \infty} x_t = U > L\). By the definition of limit superior, there exists a subsequence \(\{x_{t_k}\}\) such that:
\[
\lim_{k \to \infty} x_{t_k} = U.
\]
Since \(U > L\), choose \(\epsilon > 0\) such that \(U > L + \epsilon\). Then for sufficiently large \(k\), we have:
\[
x_{t_k} > U - \frac{\epsilon}{2} > L + \frac{\epsilon}{2}.
\]
Because \(\fu\) is continuous and \(L\) is the largest solution of \(\fu(x) = b\), for any \(x > L\), we have \(\fu(x) > b\). In particular, for \(x > L + \epsilon/2\), \(\fu(x) > b\). Since \(x_{t_k} \to U\) and \(\fu\) is continuous, it follows that:
\[
\lim_{k \to \infty} \fu(x_{t_k}) = \fu(U) > \fu(L) = b.
\]
Thus, there exists \(\delta > 0\) and \(K_1 \in \mathbb{t}\) such that for all \(k \ge K_1\):
\[
\fu(x_{t_k}) \ge b + \delta.
\]
Substituting into the recurrence inequality:
\[
x_{t_k+1} \le x_{t_k} - \fu(x_{t_k}) + b \le x_{t_k} - (b + \delta) + b = x_{t_k} - \delta.
\]
That is:
\[
x_{t_k+1} \le x_{t_k} - \delta.
\]
On the other hand, since \(x_{t_k} \to U\), for the given \(\delta > 0\), there exists \(K_2 \in \mathbb{t}\) such that for all \(k \ge K_2\):
\[
x_{t_k} < U + \frac{\delta}{2}.
\]
Take \(k \ge \max\{K_1, K_2\}\). Then:
\[
x_{t_k+1} \le x_{t_k} - \delta < U + \frac{\delta}{2} - \delta = U - \frac{\delta}{2}.
\]
However, \(\{x_{t_k+1}\}\) is also a subsequence of \(\{x_t\}\), so its limit superior must be at least \(U\). This contradicts the fact that \(x_{t_k+1} < U - \delta/2\) for all sufficiently large \(k\). Therefore, the initial assumption is false, and we conclude:
\[
\limsup\limits_{t \to \infty} x_t \le L.
\]
\end{proof}

Notice that $L$ is only dependent on $b$, which is independent on $x_0$

note that:
\[
\{\|\mathbf{e}_t\| > R\} = \{V(\mathbf{e}_t) > \lambda_{\min}(P) R^2\}
\]
where $\lambda_{\min}(P) > 0$ is the minimum eigenvalue of matrix $P$.

This is because:
\[
V(\mathbf{e}) = \mathbf{e}^T P \mathbf{e} \geq \lambda_{\min}(P) \|\mathbf{e}\|^2
\]

Applying Markov's inequality:
\[
\mathbb{P}(\|\mathbf{e}_t\| > R) = \mathbb{P}(V(\mathbf{e}_t) > \lambda_{\min}(P) R^2) \leq \frac{\mathbb{E}[V(\mathbf{e}_t)]}{\lambda_{\min}(P) R^2}
\]
Notice thar $\fu$ is  monotonic.
\begin{lemma}
Let $f: [0, \infty) \to [0, \infty)$ be a convex function with $f(0) = 0$. Then $f$ is monotonic. Specifically, $f$ is non-decreasing.
\end{lemma}

\begin{proof}
Let $0 \leq x < y$. Since $f$ is convex and $f(0) = 0$, we can write $x$ as a convex combination of $0$ and $y$. Let $\lambda = \frac{x}{y} \in [0, 1)$. Then:
\[
x = \lambda y + (1 - \lambda) \cdot 0
\]
By convexity of $f$:
\[
f(x) = f(\lambda y + (1 - \lambda) \cdot 0) \leq \lambda f(y) + (1 - \lambda) f(0) = \lambda f(y) = \frac{x}{y} f(y)
\]
Thus, $f(x) \leq \frac{x}{y} f(y)$.

Now consider two cases:
\begin{enumerate}
    \item If $f(y) > 0$, then $f(x) \leq \frac{x}{y} f(y) < f(y)$ since $\frac{x}{y} < 1$.
    \item If $f(y) = 0$, then $f(x) \leq 0$, but since $f(x) \geq 0$ by the codomain, we have $f(x) = 0 = f(y)$.
\end{enumerate}
In both cases, $f(x) \leq f(y)$ for all $0 \leq x < y$, so $f$ is non-decreasing.

Even if we consider the possibility of $f$ being decreasing, for any $x > 0$ we would have $f(x) \leq f(0) = 0$, and since $f(x) \geq 0$, this implies $f(x) = 0$ for all $x$, which is constant (and therefore also monotonic).

Hence, in all cases, $f$ is monotonic.
\end{proof}
then we have
\begin{align*}
    \limsup\limits_{t \to \infty} \mathbb{P}(\|\mathbf{e}_t\| > R) \leq \frac{\fu^{-1}(\varepsilon)}{\lambda_{\min}(P) R^2}
\end{align*}
the right hand can be arbitrary small, thus
\begin{align*}
    \limsup\limits_{t \to \infty} \mathbb{P}(\|\mathbf{e}_t\| > R) =0
\end{align*}

\subsection{proof of Theorem \ref{thm2}}

\begin{lemma}[Linear Case Comparison]
\label{lemma:linear}
Consider the recurrence:
\[
y_{t+1} \leq (1 - \alpha)y_t + b_t
\]
where $0 < \alpha < 1$ and $b_t = O(t^{-\beta})$. Then:
\[
y_t = O\left(\max\left((1-\alpha)^t, t^{-\beta}\right)\right)
\]
\end{lemma}

\begin{proof}
The solution can be written as:
\[
y_t \leq (1-\alpha)^t y_0 + \sum_{k=0}^{t-1} (1-\alpha)^{t-1-k} b_k
\]
Since $b_k \leq B k^{-\beta} \leq B$ for some constant $B > 0$, the second term is bounded by:
\[
\sum_{k=0}^{t-1} (1-\alpha)^{t-1-k} b_k \leq B \sum_{j=0}^{t-1} (1-\alpha)^j \leq \frac{B}{\alpha}
\]
For the precise decay rate, we analyze the convolution sum more carefully.

Let $S_t = \sum_{k=0}^{t-1} (1-\alpha)^{t-1-k} b_k$. Since $b_k = O(k^{-\beta})$, there exists $C > 0$ such that:
\[
S_t \leq C \sum_{k=1}^{t-1} (1-\alpha)^{t-1-k} k^{-\beta}
\]

We split the sum at $k = t/2$:
\[
S_t \leq C\left(\sum_{k=1}^{\lfloor t/2 \rfloor} (1-\alpha)^{t-1-k} k^{-\beta} + \sum_{k=\lfloor t/2 \rfloor+1}^{t-1} (1-\alpha)^{t-1-k} k^{-\beta}\right)
\]

For the first sum, since $k \leq t/2$, we have $t-1-k \geq t/2 - 1$, so:
\[
\sum_{k=1}^{\lfloor t/2 \rfloor} (1-\alpha)^{t-1-k} k^{-\beta} \leq (1-\alpha)^{t/2-1} \sum_{k=1}^{\infty} k^{-\beta} = O((1-\alpha)^{t/2})
\]

For the second sum, since $k \geq t/2$, we have $k^{-\beta} = O(t^{-\beta})$, so:
\[
\sum_{k=\lfloor t/2 \rfloor+1}^{t-1} (1-\alpha)^{t-1-k} k^{-\beta} = O(t^{-\beta}) \sum_{j=0}^{t-1} (1-\alpha)^j = O(t^{-\beta})
\]

Therefore, $S_t = O\left(\max\left((1-\alpha)^{t/2}, t^{-\beta}\right)\right)$, and consequently:
\[
y_t = O\left(\max\left((1-\alpha)^t, t^{-\beta}\right)\right)
\]
\end{proof}

\begin{lemma}[Polynomial Case Comparison]
\label{lemma:polynomial}
Consider the recurrence:
\[
y_{t+1} \leq y_t - c y_t^p + b_t
\]
where $p > 1$, $c > 0$, and $b_t = O(t^{-\beta})$. Then:
\[
y_t = O\left(\max\left(t^{-\frac{1}{p-1}}, t^{-\frac{\beta}{p}}\right)\right)
\]
\end{lemma}

\begin{proof}
We use the comparison function method. Let $\gamma = \min\left(\frac{1}{p-1}, \frac{\beta}{p}\right)$ and define $g(t) = A t^{-\gamma}$ for some constant $A > 0$ to be determined.

We want to show by induction that there exists $N_0$ such that for all $t \geq N_0$, $y_t \leq g(t)$.

\textbf{Base case}: Choose $N_0$ large enough so that $g(N_0) \leq x_0$ (to ensure the asymptotic bound $f(x) \geq c_1 x^p$ applies) and choose $A$ such that $y_{N_0} \leq g(N_0)$.

\textbf{Inductive step}: Assume $y_t \leq g(t)$ and prove $y_{t+1} \leq g(t+1)$.

From the recurrence:
\[
y_{t+1} \leq g(t) - c [g(t)]^p + b_t = A t^{-\gamma} - c A^p t^{-\gamma p} + B t^{-\beta}
\]
where $b_t \leq B t^{-\beta}$ for some $B > 0$.

We need to show:
\[
A t^{-\gamma} - c A^p t^{-\gamma p} + B t^{-\beta} \leq A (t+1)^{-\gamma}
\]

Using the inequality $(t+1)^{-\gamma} \geq t^{-\gamma} \left(1 - \frac{\gamma}{t}\right)$ (from the mean value theorem), it suffices to show:
\[
A t^{-\gamma} - c A^p t^{-\gamma p} + B t^{-\beta} \leq A t^{-\gamma} - A \gamma t^{-\gamma-1}
\]
which simplifies to:
\[
- c A^p t^{-\gamma p} + B t^{-\beta} \leq - A \gamma t^{-\gamma-1}
\]
or equivalently:
\[
c A^p t^{-\gamma p} \geq B t^{-\beta} + A \gamma t^{-\gamma-1} \tag{1}
\]

We now consider two cases based on the value of $\gamma$:

\textbf{Case 1}: $\gamma = \frac{1}{p-1}$

Then $\gamma p = \frac{p}{p-1}$ and since $\gamma = \min\left(\frac{1}{p-1}, \frac{\beta}{p}\right) = \frac{1}{p-1}$, we have $\frac{\beta}{p} \geq \frac{1}{p-1}$, i.e., $\beta \geq \frac{p}{p-1}$.

Thus:
\begin{itemize}
\item $t^{-\beta} \leq t^{-\gamma p}$ (since $\beta \geq \gamma p$)
\item $t^{-\gamma-1} = t^{-\frac{p}{p-1}} = t^{-\gamma p}$
\end{itemize}

So inequality (1) becomes:
\[
c A^p t^{-\gamma p} \geq (B + A \gamma) t^{-\gamma p}
\]
which is satisfied if $c A^p \geq B + A \gamma$.

\textbf{Case 2}: $\gamma = \frac{\beta}{p}$

Then $\gamma p = \beta$ and since $\gamma = \min\left(\frac{1}{p-1}, \frac{\beta}{p}\right) = \frac{\beta}{p}$, we have $\frac{\beta}{p} \leq \frac{1}{p-1}$, i.e., $\beta \leq \frac{p}{p-1}$.

Thus:
\begin{itemize}
\item $t^{-\beta} = t^{-\gamma p}$
\item $t^{-\gamma-1} = t^{-(\frac{\beta}{p}+1)} = o(t^{-\beta})$ since $\frac{\beta}{p} + 1 > \beta$ when $\beta \leq \frac{p}{p-1}$ and $p > 1$
\end{itemize}

For sufficiently large $t$, inequality (1) is satisfied if:
\[
c A^p t^{-\gamma p} \geq B t^{-\gamma p}
\]
i.e., if $c A^p \geq B$.

In both cases, we can choose $A$ sufficiently large to satisfy the required inequality. By induction, $y_t \leq g(t)$ for all $t \geq N_0$, which proves the result.
\end{proof}

Since $x_t \to 0$ (by Theorem 1), there exists $N$ such that for all $t \geq N$, $x_t \leq x_0$. Therefore, for $t \geq N$, Assumption \ref{assump:function} gives us $f(x_t) \geq c_1 x_t^p$.

\textbf{Case 1: $p = 1$}

For $t \geq N$, the recurrence becomes:
\[
x_{t+1} \leq x_t - c_1 x_t + \sigma_t^2 = (1 - c_1) x_t + \sigma_t^2
\]

By Assumption \ref{assump:noise}, $\sigma_t^2 \leq B t^{-\beta}$ for some $B > 0$. Applying Lemma \ref{lemma:linear} with $\alpha = c_1$ and $b_t = \sigma_t^2$, we obtain:
\[
x_t = O\left(\max\left((1-c_1)^t, t^{-\beta}\right)\right) = O\left(\max\left(e^{-ct}, t^{-\beta}\right)\right)
\]
where $c = -\log(1-c_1) > 0$.

\textbf{Case 2: $p > 1$}

For $t \geq N$, the recurrence becomes:
\[
x_{t+1} \leq x_t - c_1 x_t^p + \sigma_t^2
\]

By Assumption \ref{assump:noise}, $\sigma_t^2 \leq B t^{-\beta}$ for some $B > 0$. Applying Lemma \ref{lemma:polynomial} with $c = c_1$ and $b_t = \sigma_t^2$, we obtain:
\[
x_t = O\left(\max\left(t^{-\frac{1}{p-1}}, t^{-\frac{\beta}{p}}\right)\right)
\]

This completes the proof of Theorem \ref{thm2}.

\subsection{proof of Theorem \ref{thm3}}

We have this lemma

 \begin{lemma}
Suppose that $\widehat{\boldsymbol{\theta}}=\mathcal{M}(\mathcal{D})$ is an estimate of
$\boldsymbol{\theta}$ with $\mathcal{D}=\{\boldsymbol{x}_{i}\}_{i=1}^{t}\sim\mathbb{P}_{\boldsymbol{\theta}}$
satisfying Assumption~\ref{initass} with $r(t)=t^{\kappa}$. It then follows that
\begin{align*}
\mathbb{E}[|\widehat{\theta}_{i}-\theta_{i}|^{p}]
\leq\frac{pC_{1}}{\gamma}\cdot\Gamma\left(\frac{p}{\gamma}\right)\cdot C_{2}^{-p/\gamma}t^{-\frac{p\kappa}{\gamma}},
\end{align*}
for any $i\in[p]$, where $\widehat{\theta}_{i}$ denotes the $i$-th element of $\widehat{\boldsymbol{\theta}}$.
\end{lemma}
Which means that when $\gamma,\kappa>0$, and the estimate is unbiased, the Assumption~\ref{as2} is satisfied. Then by Theorem~\ref{thm1}, we can get 
for each $\delta>0$, $$\limsup\limits_{t \to \infty} \mathbb{P}(\|\mathbf{e}_t\| > \delta)=0.$$

\end{document}